\documentclass{SCIS2021}

\usepackage{bm}
\usepackage{graphicx}
\usepackage{subfig}

\newcommand{\expect}{\mathbb{E}}

\newcommand{\indicate}{\mathbb{I}}

\newtheorem{thm}{Theorem}
\newtheorem{myDef}{Definition}
\newtheorem{prop}[thm]{Proposition}

\begin{document}
\ArticleType{RESEARCH PAPER}
\Year{2020}
\Month{}
\Vol{}
\No{}
\DOI{}
\ArtNo{}
\ReceiveDate{}
\ReviseDate{}
\AcceptDate{}
\OnlineDate{}

\title{Understanding Adversarial Attacks on Observations\\ in Deep Reinforcement Learning}{Title keyword 5 for citation Title for citation Title for citation}


\author[1]{You Qiaoben}{}
\author[1]{Chengyang Ying}{}
\author[1]{Xinning Zhou}{}
\author[1,2]{Hang Su}{}
\author[1,2]{Jun Zhu}{dcszj@mail.tsinghua.edu.cn}
\author[1]{Bo Zhang}{}

\AuthorMark{Qiaoben Y}

\AuthorCitation{Qiaoben Y, Ying C Y, Zhou X N, et al}


\address[1]{Department of Computer Science and Technology,\\Beijing National Research Center for Information Science and Technology, \\ Tsinghua-Bosch Joint Center for Machine Learning, Institute for Artificial Intelligence, \\
Tsinghua University, Beijing 100084, China}
\address[2]{ Peng Cheng Laboratory, Shenzhen, Guangdong, 518055, China}

\abstract{Deep reinforcement learning models are vulnerable to adversarial attacks that can decrease a victim's cumulative expected reward by manipulating the victim's observations. Despite the efficiency of previous optimization-based methods for generating adversarial noise in supervised learning, such methods might not be able to achieve the lowest cumulative reward since they do not explore the environmental dynamics in general. 
In this paper, we provide a framework to better understand the existing methods by reformulating the problem of adversarial attacks on reinforcement learning in the function space. Our reformulation generates an optimal adversary in the function space of the targeted attacks, repelling them via a generic two-stage framework. In the first stage, we train a deceptive policy by hacking the environment, and discover a set of trajectories routing to the lowest reward or the worst-case performance. Next, the adversary misleads the victim to imitate the deceptive policy by perturbing the observations. Compared to existing approaches, we theoretically show that our adversary is stronger under an appropriate noise level.
Extensive experiments demonstrate our method's superiority in terms of efficiency and effectiveness, achieving the state-of-the-art performance in both Atari and MuJoCo environments.}


\maketitle

\section{Introduction}
The increasing sophistication and ubiquity of reinforcement learning has resulted in impressive performance in Atari games~\cite{1,2,3} and numerous other tasks~\cite{4,5}. But the performance remains vulnerable when an adversary accesses inputs to a reinforcement learning (RL) system, and implements malicious attacks to deceive a deep policy~\cite{6,7,8,9}. A deep reinforcement learning (DRL) agent may thus take sub-optimal (or even harmful) actions so that the performance of a trained RL agent is degraded. As the RL-based frameworks become increasingly widely deployed in real-world scenarios, it becomes an indispensable prerequisite to understand adversarial attacks against DRL policies, especially for safety-related or life-critical applications such as industrial robots and self-driving vehicles.

\begin{figure}[!t]
\centering
\includegraphics [width=1\linewidth]{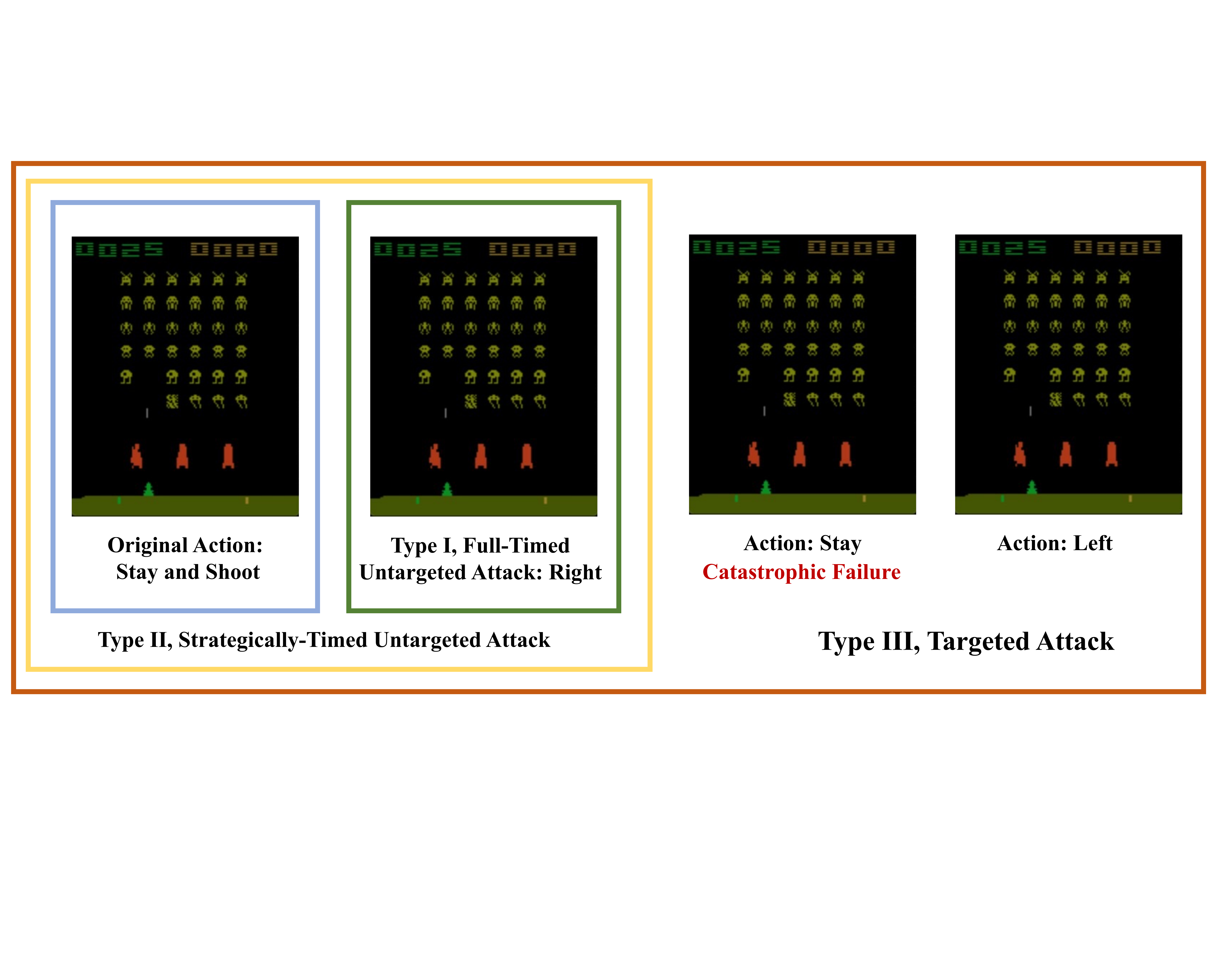}
\caption{The actions taken by the victim policy under an adversarial attack against Atari. The optimal action for the agent in the current state is to ``stay and shoot''---and get the bonus; while the worst action is to ``stay'' which can lead to the death of the agent. Most of the early methods can be categorized into the full-timed or strategically-timed untargeted attacks, which may not achieve the worst case action since they usually do not infer the MDP's dyanmics. The more recent works explore the targeted attacks which lead to the  catastrophic consequences. In this paper, we provide a comprehensive understanding for the present methods with a generic framework in the function space and identify the limitations of these methods.
}
\label{fig:example_atari}
\end{figure}

Though adversarial robustness has been extensively studied in supervised learning~\cite{10,11,12,13,14}, the adversarial vulnerability of DRL has been less investigated~\cite{6,9,15}. Ever since the seminal work~\cite{16}, there have been only a few papers about this problem. In DRL, the adversary can choose different part to attack, i.e. observation, transition, reward and so on. In this work, we mainly focus on adversarial attacks on observations, i.e., the perturbation only depends on the current state and would not change the state-action transition. In particular, most of the previous optimization-based methods conducted adversarial attacks following a supervised manner that aims to change the behaviors of an agent directly~\cite{6}. However, as illustrated in Fig.~\ref{fig:example_atari}, for the same observation, different types of attacks will mislead the agent to take different actions. Therefore, it may not minimize the accumulated reward to achieve the optimal attack on RL since the adversary does not infer the environmental dynamics. Recent learning-based methods~\cite{9,17} try to generate an ``optimal'' adversary by learning a global mapping from the original states to the adversarial states, which provides a possible solution to obtain the worst case agent reward. However, these methods may suffer from the inefficiency due to the high-dimensional state representations (e.g., the state dimension is $>$ 6,000 in Atari). Therefore, it is imperative to investigate \emph{whether we can effectively and efficiently generate adversarial noise by leveraging the adversarial attacks to assessing the robustness of reinforcement learning?} 

In this paper, we study the optimization-based methods due to their efficient implementation and potential in exploring the MDP's dynamics. In order to better understand the pros and cons of the current methods, we first categorize them by considering the manner of noise generation and attack frequency of these methods. 
The first type is the full-timed untargeted attack which misleads the agent to a sub-optimal action~\cite{6,8}. The second type is a more covert and low-frequency attack of  strategically-timed attack, which only attacks the agent at a small subset of time steps~\cite{18}. It is noted that both types of such methods try to mislead the agent to a sub-optimal action rather than the worst-case action, which is referred to as untargeted attack in this paper. The third type lures the agent to a specified malicious policy using a targeted attack~\cite{8,19,20} which allows the adversary to choose when and how to attack by adjusting the target action.

Then, we present a reformulation of these attacks in the function space, where each type corresponds to a particular function space. We show that the targeted attacks are much more stronger than the untargeted attacks, i.e., the space corresponding to the targeted attacks is larger than the space for the untargeted attacks. After close examinations, we further find that all the previous targeted attacks follow a pessimistic assumption on adversary, i.e., the performance of the attacked policy is closed to the victim policy, which may therefore hinder the performance. To address this issue, we introduce a more reasonable optimistic assumption that the performance of the attacked policy should be close to the the worst-case agent which allows the adversary to explore a much larger space and generate more powerful adversary. 

Following the comprehensive analysis on the present methods, we propose a novel optimization-based method that aims to minimize the expected accumulated reward for the victim policy under an alternative optimistic assumption. To the best of our knowledge, our method is the first optimization-based method to consider the dynamic of the MDP by approximately estimating the performance of the attacked policy.
Specifically, we reformulate the task as a two-stage optimization problem by introducing an intermediate \emph{deceptive policy} to explore the environment, which serves as a targeted policy to achieve the worst case agent. Instead of maximizing the least preferred action's Q-values~\cite{20}, our method generates adversarial perturbations by minimizing the KL-divergence between the deceptive policy and the victim policy, which can therefore satisfy our optimistic assumption. 

Furthermore, we provide a theoretical analysis and show that our algorithm results in a tighter upper-bound for the attacked agent's performance than other existing adversary under an appropriate noise level. Extensive experiments on both MuJoCo and Atari environments show that our attacks achieve state-of-the-art performance with the lowest rewards in the vast majority of the environments, compared with strong optimization-based and learning-based baselines.
Furthermore, our algorithm is the first optimization-based method to achieve nearly the lowest reward in several Atari environments by exploring the MDP's dynamics, and is much more efficient than the learning-based methods.

\section{Preliminaries}
We first briefly review the recent adversarial attack algorithms for deep RLs and present the framework of the state-adversarial Markov decision process (SA-MDP).

\subsection{Adversarial Attacks in DRL}
\label{section:adversarial attacks in DRL}
In DRL, the adversary can choose different part to attack, like observations and environment dynamics. Modifying observations and dynamics have structurally different impact to the agent and some current work~\cite{34} hopes to find the connection between them. Moreover, since the adversary often does not have authority to modify the original state in a simulator, in this paper, we mainly focus on the setting that perturbs the observations. 
In general, the present methods can be categorized into optimization-based and learning-based attacks. At each state, the adversary in optimization-based attacks generates a perturbation with the optimization-based adversarial attack in supervised learning.
Based on the manner of noise generation and attack frequency, we can further classify the optimization-based adversary into three categories. The first type~\cite{6} applies an untargeted attack algorithm to mislead the agent to choose a sub-optimal action at each time step, while the second type provides heuristic adversaries that attack the agent on a subset of time steps, using a solver of an untargeted attack~\cite{7,18}. The third type of adversaries mislead the agent to some heuristic target actions by leveraging the Q-value functions~\cite{8, 19, 20}. 
In a learning-based attack, the adversary learns a mapping and directly generates a perturbation with this mapping.
A recent example is \cite{9}, which provides a function approximator to learn the perturbation under the framework~\cite{8} with superior performance over optimization-based methods in MuJoCo environments. However, the previous learning-based methods can be inefficient in high-dimensional environments.

\subsection{State-Adversarial Markov Decision Process}
Zhang et al.~\cite{8,9} present a unified framework of a state-adversarial Markov decision process (SA-MDP), which is a modified MDP for the perturbation on state observations. Formally, let $\mathcal{S}$ be the state set and $\mathcal{F}(\mathcal{S})$ be the set of all distributions on $\mathcal{S}$. SA-MDP has an attacker set $G$, and any attacker $g: \mathcal{S} \rightarrow \mathcal{F}(\mathcal{S})$ can perturb the state observation which is configured as $s$ to $\bar{s} \sim g(\cdot|s)$, where $g(\cdot|s)$ is the distribution of the perturbed state.
In particular, SA-MDP can be represented as a 6-tuple as $\mathcal{M} = (\mathcal{S}, \mathcal{A},$ $ \mathcal{B}, P_a, \mathcal{R}, \gamma)$, where $\mathcal{A}$ is the action set, and $\mathcal{B}(s)$ is the disturbed state set (usually small) around state $s$; 
$P_a$ is a transition function, $R$ is a reward function, and $\gamma$ is a discount factor.  An agent acts following a policy $\pi$. In SA-MDP, the agent takes the action as 
\begin{equation}
    \pi_g(a|s) \equiv \sum_{\bar{s} \in S}g(\bar{s}|s)\pi(a|\bar{s}), 
\end{equation}
where $\pi$ is the victim policy to be attacked and $g(\bar{s}|s)$ is the mapping for the orginal state to the adversarial state. 
The adversary aims to minimize the expected total reward of $\pi$ by applying the perturbations as
\begin{equation}
\label{equal:SA-MDP}
       g^* =\mathop{\arg\min}_{ g \in G} \expect_{a_t \sim \pi_g(\cdot|s_t)}\left[ \sum_{t=0}^{\infty} \gamma^t r_t \right].
 \end{equation}
For notation simplicity, we omit the initial state distribution, denote $a \sim \pi_g$ instead of $a_t \sim \pi_g(\cdot|s_t)$ in this paper.

\section{Understanding Adversarial Attacks in Function Spaces}
We first provide a reformulation of state-adversarial Markov decision process in the function space, and then show that the existing three types of optimization-based adversarial algorithm can be characterized as different subspaces; finally, we provide a comprehensive analysis on the function space which motivates us to find a better adversary.

\subsection{Reformulating SA-MDP in Function Spaces}
\label{section: Reformulate SA-MDP in function space}
In order to provide an in-depth understanding of the current methods, we reformulate SA-MDP in the function space of $\mathcal{H}$ by classifying the existing optimization-based attack algorithms into different types. As illustrated in Fig.~\ref{fig:chart}, given a pre-trained policy $\pi$ as the victim policy, 
the adversary aims to minimize the expected total reward of $\pi$ by applying the perturbations $\delta_{s_t}$ to state $s_t$ as $h(s_t) = s_t + \delta_{s_t}$. 
Let $\mathcal{H}$ be the space of the adversary's function as
\begin{equation}
    \mathcal{H} = \{h \big| \|h(s)-s\|_p \leq \epsilon,~\forall s \in \mathcal{S} \},
\end{equation} 
where the constant $\epsilon$ ifs the level of the $l_p$-norm adversarial noise that measures the ability of an adversary. 
We can reformulate problem~\eqref{equal:SA-MDP} by solving the optimal function in the function space as $h^* \in \mathcal{H}$. Suppose a function as $\pi_h : \mathcal{S} \rightarrow \mathcal{F}(\mathcal{A})$ such that
\begin{equation}
    \pi_h(a|s) \equiv \pi(a|h(s)),
\end{equation}
which means an agent with policy $\pi_h(a|s)$ behaves similarly to the victim agent with policy $\pi$ when the observed state is perturbed to $h(s)$. With our setting of attacker $h$ and function space $\mathcal{H}$, we can derive the optimal function set $\mathcal{H}^*$ to generate the adversarial perturbation by rewriting problem~\eqref{equal:SA-MDP} 
\begin{equation}
    \label{origin}
     \mathcal{H}^*= \Bigg\{h^\ast\ \bigg| h^* = \mathop{\arg\min}_{h \in H} \left[R(\pi_h)\triangleq \expect_{a \sim \pi_h}  \sum_{t=0}^{\infty}\gamma^tr_t
         \right]\Bigg\}
\end{equation}
where $R(\pi_h)$ is the expected total reward. 

\begin{figure}[!t]
\centering
\includegraphics[width=1\linewidth]{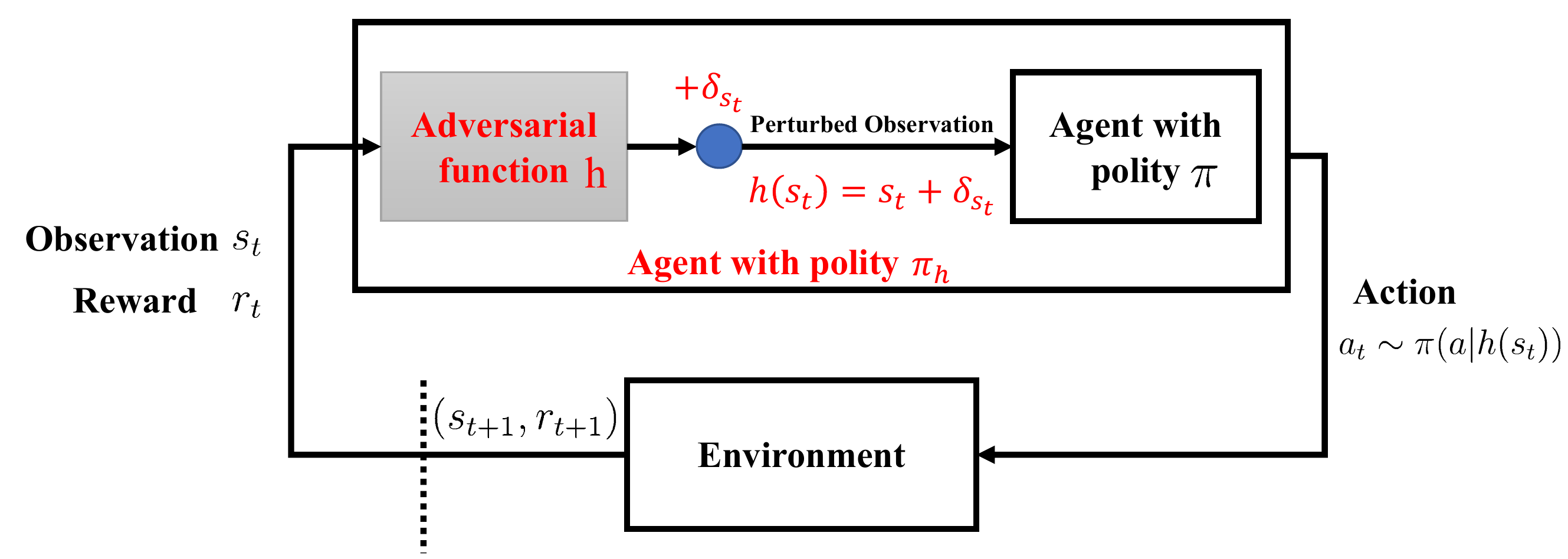}
\caption{Adversarial attacks on the victim's observation: Given an agent with policy $\pi$, the attacker applies an adversarial function $h$ to the observation. The victim's observation state is perturbed to $h(s_t)=s_t+\delta_{s_t}$, yielding a sub-optimal action. The victim agent consequentially behaves as $a_t\sim \pi(a|h(s_t))$ which can be recognized as an agent with adversarial policy  $a_t\sim \pi_h(a|s_t)$.}
\label{fig:chart}
\end{figure}

\subsection{Categorizing Adversaries in Function Spaces}
\label{section: 3-2}
In this part, we classify the existing optimization-based methods with three subspaces in the function space $\mathcal{H}$ as in Table~\ref{table-surmary}. 
The first two types of subspaces refer to the untargeted adversaries that generate adversarial noise by only considering the current states. 
In the third subspace, we consider an adversary that is able to generate adversarial noise by specifying the targeted action or policy at each time step. 

\begin{table*}[]
\centering

\small
\begin{tabular}{|c|c|c|c|}
\hline
Function space & Attack Manner &  Attack Frequency & Methods\\
\hline
$\mathcal{H}_1$   &        Untargeted Attack & Full-Timed & Huang et al.~\cite{6}; Zhang et al.~\cite{8} \\
\hline
$\mathcal{H}_2$  &  Untargeted Attack & Strategically-Timed& Kos and Song~\cite{7}; Lin et al.~\cite{18}; \\

&&&Sun et al.~\cite{21};Yang et al.~\cite{22} \\
\hline
$\mathcal{H}_3$ & Targeted Attack & Full-Timed &Pattanaik et al.~\cite{19}; Lin et al.~\cite{20}; \\
&&&Zhang et al.~\cite{8} \\
\hline
\end{tabular}
\caption[c]{The typical existing optimization-based attacks in deep reinforcement learning are categorized into three types based on algorithm for noise generation and the attack frequency, which are corresponding to three function spaces.}
\label{table-surmary}
\end{table*}

\subsubsection{The Function Space of Full-Timed Untargeted Attack of $\mathcal{H}_1$}
We start by introducing the adversary set $\mathcal{H}_1$, which misleads the agent to take sub-optimal actions at each state $s$ by using untargeted adversaries as
\begin{myDef}
\label{def-1} 
Assume $\Phi_u$ is the set of the untargeted attack algorithms, for $\forall\phi\in\Phi_u$,  $\mathcal{H}_1$ is the function space to represent the adversary $h \in \mathcal{H}$ which generates the perturbed state for each $\phi\in\Phi_u$ as
\begin{equation}
    \mathcal{H}_1 = \{ h \in \mathcal{H} |h(s) \equiv \phi_{\pi}^{\epsilon}(s),~\exists \phi \in \Phi_u,~\forall s \in \mathcal{S} \}, \nonumber 
\end{equation}
where $\phi_{\pi}^{\epsilon}$ 
is an instantiated mapping from the original state space $\mathcal{S}$ to the adversarial state space $\mathcal{S}$ parameterized by the noise level $\epsilon$ and the victim policy $\pi$, i.e., $\|\phi_{\pi}^{\epsilon}(s) - s\|_p \leq \epsilon$ holds for any state $s$.
\end{myDef}
As an untargeted attack algorithm, $\phi\in \Phi_u$ maximizes the divergence between the victim policy $\pi$ and the attacked policy $\pi_h$ at state $s$.
For instance, as a representative work in $\mathcal{H}_1$~\cite{6}, the distance is evaluated by the cross-entropy loss between the attacked policy and the policy that takes the action $a$ with the maximum probability of the victim policy. 
\cite{8} further improve the methods by maximizing the KL divergence between the victim policy and the attacked policy with projected gradient descent (PGD)~\cite{13} or convex relaxations~\cite{23} of neural networks for each state.

However, in some environments, any adversary in $\mathcal{H}_1$ may not be in $\mathcal{H}^*$, i.e., one cannot achieve the worst case of expected accumulated reward. We present Proposition 1 to show the weakness of $\mathcal{H}_1$.

\begin{prop}[\textbf{Weakness of} \bm{$\mathcal{H}_1$}]
\label{thm-1}
There exists an MDP such that $\forall \pi \in \Pi^*$, the optimal function $h^\ast$ is not included in the space of $\mathcal{H}_1$, i.e., $~\mathcal{H}^* \cap \mathcal{H}_1 = \varnothing$,  where $\Pi^*$ is the optimal policy set that can maximize the accumulated reward for an agent.
\end{prop}

Proposition~\ref{thm-1} is apparent and we provide a constructed example in Fig.~\ref{counter-example} for illustration. The agent gains reward $-1$ in the grey state, gains reward $+1$ in the red state and $0$ in other states. The grey state and the red state are the terminal states. The agent can only move right or down in the left part. Then the optimal policy (red arrow) $\pi^* \in \Pi^*$ must reach the green state $s_g$ and move right with $a^*$.
By Def.~\ref{def-1}, the full-timed untargeted attacks aim to change the original actions as $$\forall h_1 \in \mathcal{H}_1, \pi_{h_1}(a^*|s_g) <1,$$
which prevents the victim from taking the minimum reward (reach grey state).
However, every worst-case policy with the flipped reward (black arrow) also reach the green state and move right, which means $$\forall h^* \in \mathcal{H}^*, \pi_{h^*}(a^*|s_g) = 1.$$ Then $V^{\pi_{h_1}}(s_g) > V^{\pi_{h^*}}(s_g)$. From this example, we can illustrate that   $~\mathcal{H}^* \cap \mathcal{H}_1 = \varnothing$.

\begin{figure}[!t]
\centering
\begin{minipage}[c]{0.48\textwidth}
\centering
\includegraphics[width=0.9\textwidth]{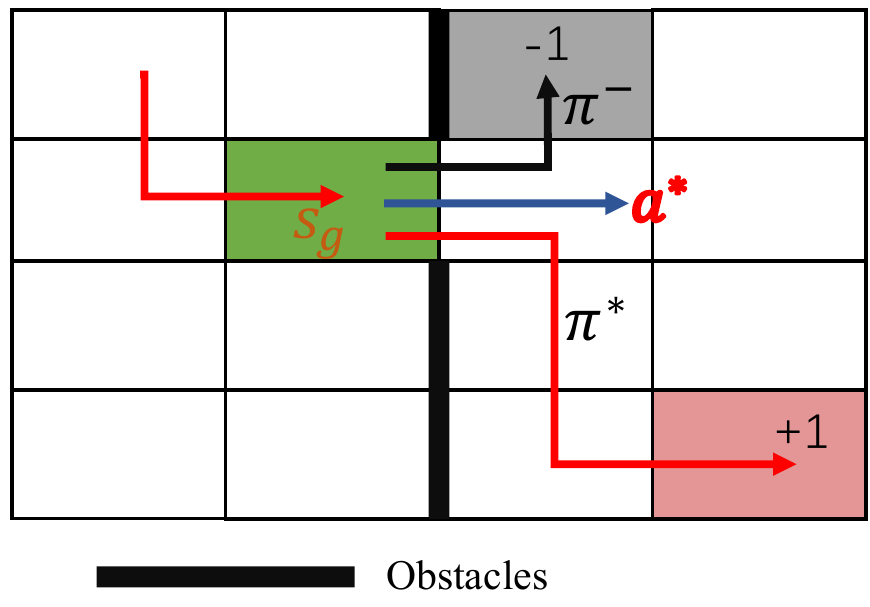}
\end{minipage}
\hspace{0.02\textwidth}
\begin{minipage}[c]{0.48\textwidth}
\centering
\includegraphics[width=0.9\textwidth]{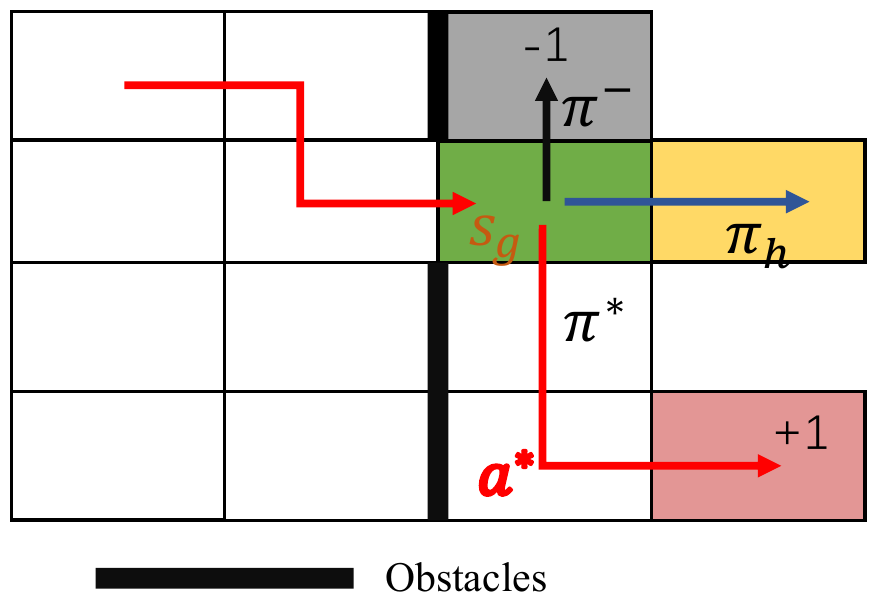}
\end{minipage}\\[3mm]
\begin{minipage}[t]{0.48\textwidth}
\centering
\caption{An example in Grid World with two obstacles for $\mathcal{H}_1$}
\label{counter-example}
\end{minipage}
\hspace{0.02\textwidth}
\begin{minipage}[t]{0.48\textwidth}
\centering
\caption{An example in Grid World with two obstacles for $\mathcal{H}_2$}
\label{counter-example-2}
\end{minipage}
\end{figure}

From Proposition~\ref{thm-1}, we can see that it is not always necessary to generate noise at each time step to minimize the expected reward of the attacked policy. The second type of strategically-timed untargeted attacks can be more effective in this aspect, though they also suffer some limitations, as detailed below.  

\subsubsection{The Function Space of Strategically-timed Untargeted Attack of $\mathcal{H}_2$}

We characterize the second type of strategically-timed untargeted attack which can choose either to perturb the observation or to maintain the origin observation. Formally, we characterize this type as the second function space $\mathcal{H}_2$ as
\begin{myDef}
$\mathcal{H}_2$ is the function space to represent the adversary $h \in \mathcal{H}$ which
only attacks the agent at a small subset of time steps strategically as 
\begin{equation}
\label{def-2}
\begin{aligned}
\mathcal{H}_2 = \big\{&h \in \mathcal{H} | h(s) \equiv \indicate_a(s)\phi_{\pi}^{\epsilon}(s) + (1-\indicate_a(s)) s ,\exists \phi \in \Phi_u,~\forall s \in \mathcal{S}\big\}, \nonumber
\end{aligned}
\end{equation}
where $\phi_{\pi}^{\epsilon}(s)$ 
is an instantiated mapping of $\phi$ parameterized by  $\epsilon$ and $\pi$ as in Definition~\ref{def-1}, and $\indicate_a(s)$ indicates whether the adversary attack at state $s$ or not, i,e., $\indicate_a(s) = 1$ when the adversary attacks at state $s$, otherwise $\indicate_a(s) = 0$. 
\end{myDef}
There exist a few attempts in this direction~\cite{7,18,21,22}, which aim at reducing the frequency of attacks rather than enhancing the attacker's capabilities. These algorithms heuristically control whether to attack based on the current state $s$. 
Although the attacked policy in $\mathcal{H}_2$ has one more option at each state, i.e., it can act as the victim policy, this adversary may still not be in $\mathcal{H}^*$. We also present Proposition 2 to show the weakness of $\mathcal{H}_2$, which is similar to Proposition 1.

\begin{prop}[\textbf{Weakness of} \bm{$\mathcal{H}_2$}]
\label{thm-2}
There exists an MDP such that $\forall \pi \in \Pi^*$, the optimal function $h^\ast$ is not included in the space of $\mathcal{H}_2$, i.e., $~\mathcal{H}^* \cap \mathcal{H}_2 = \varnothing$,  where $\Pi^*$ is the optimal policy set that can maximize the accumulated reward for an agent.
\end{prop}

\begin{proof}
Suppose there exist $|\Phi_u|$ attack algorithms. We provide a constructed example in Fig.~\ref{counter-example-2} for $|\Phi_u|=1$, without losing generality.
Since strategically-timed attacked policy (in Definition 2) is not related to the worst state in the environment, there exists an MDP in Proposition~\ref{thm-2}.
According the definition~\ref{def-2}, there are three possible strategically-timed attacked policy at green state $s_g$: $\{$right, down$\}$, $\{$up, down$\}$, or $\{$left, down$\}$
In Fig.~\ref{counter-example-2}, we suppose that the strategically-timed attacked policy moves right or down in green state $s_g$. The agent gains reward $+1$ in the red state and $0$ in other states. The agent gains reward $-1$ in the grey state. The grey state and the red state are the terminal states. Then the optimal policy (red arrow) $\pi^* \in \Pi^*$ must reach the green state $s_g$ and move right with $a^*$, i.e, $V^{\pi_{h_2}}(s_g) \geq 0$.
However, every worst-case policy with the adversarial reward (black arrow) also reach the grey state, i.e, $V^{\pi_{h_2}}(s_g) > V^{\pi_{h^*}}(s_g)=-1$. From this example, we can illustrate that   $~\mathcal{H}^* \cap \mathcal{H}_2 = \varnothing$. When the strategically-timeed attacked policy at $s_g$ is $\{$top, down$\}$ or $\{$left, down$\}$, we can give another example by simply setting the yellow state as the worst case with reward $-1$.
\end{proof}

\subsubsection{The Function Space of Targeted Attack of $\mathcal{H}_3$}
Compared to the untargeted attack algorithms, the targeted attack algorithms allows flexibility in specifying the target action or the target policy, which is categorized as the third function space $\mathcal{H}_3$.
\begin{myDef}
\label{def-3}
Assume $\Phi_t$ is the set of all targeted attack algorithms, for $\forall\phi\in\Phi_t$, $\mathcal{H}_3$ is the function space to represent the adversary which generates the perturbed state $\phi_{\pi, \pi'}^{\epsilon}(s)$ at each state $s$ as 
\begin{equation}
\begin{aligned}
 \mathcal{H}_3 = \{ &h \in \mathcal{H} |h(s) \equiv \phi_{\pi, \pi'}^{\epsilon}(s), ~\exists \pi' \in \Pi_{adv},~\forall s \in \mathcal{S} \}, \nonumber\\
\end{aligned}
\end{equation}
where $\phi_{\pi, \pi'}^{\epsilon}$ is an instantiated mapping of $\phi$ parameterized by the noise level $\epsilon$, the victim policy $\pi$ and the target policy $\pi'$ as in definition~\ref{def-1}, and 
$\Pi_{adv}$ is a policy set that is accessible to the adversary. 
\end{myDef}

As a targeted attack, the algorithm $\phi$ aims at finding an adversarial example with the adversarial noise level $\epsilon$ to minimize the distance between $\pi'(\cdot|s)$ and $\pi_h(\cdot|s)$. $\pi'$ belongs to $\Pi_{adv}$ and $\Pi_{adv}$ is accessible to the adversary without the adversarial optimizer. For example, $\exists \pi^* \in \Pi^*$, the optimal policy $\pi^*$ satisfies: $\pi^* \in \Pi_{adv}$ when the environment is accessible to the adversary. Theoretically, it's also worth noting that the adversary can access a worst policy $\pi^- \in \Pi^-$ when the adversary can explore the environment with the flipped reward.

\subsection{Discussion on the Function Spaces}
\label{sec: Discussion on Function Space}

In this section, we make a comprehensive analysis by finding the limitations of the present methods and providing insights on where and how to find optimal adversary $h^*$ as follows. 

\subsubsection{The Function Space $\mathcal{H}_3$ is Generally Larger Than the Function Space $\mathcal{H}_1$ and $\mathcal{H}_2$}
We first investigate the advantages of the targeted adversaries. If $\pi_{\mathcal{H}_1}$, $\pi_{\mathcal{H}_2}$ and $\pi_{\mathcal{H}_3}$ are the policies set in the corresponding function space, we find that 
\begin{thm}[\textbf{The Advantage of} \bm{$\mathcal{H}_3$}]
\label{thm-3} With an appropriate noise level $\epsilon$ and policy set $\Pi_{adv}$, the policy set satisfy $\pi_{\mathcal{H}_1} \subseteq \pi_{\mathcal{H}_2} \subseteq \pi_{\mathcal{H}_3}$.
\end{thm}
 \begin{proof}
 Fig.~\ref{fig:relation of set} shows the relationship among these policy sets. Obviously, as $\mathcal{H}_1 \subseteq \mathcal{H}_2$, we always have $\pi_{\mathcal{H}_1} \subseteq \pi_{\mathcal{H}_2}$. Without the limitation on noise level $\epsilon$ and policy set $\Pi_{adv}$, $\pi_{\mathcal{H}_3}$ is equal to $\Pi_{adv} = \Pi = \pi_{\mathcal{H}}$. Therefore, we can get $\pi_{\mathcal{H}_2} \subseteq \Pi = \pi_{\mathcal{H}_3}$.
 \end{proof}

\begin{figure}[!t]
\centering
\centering

\subfloat[$\epsilon$ = 0, $\forall \Pi_{adv} \subseteq \Pi$]{\includegraphics[width = .45\linewidth]{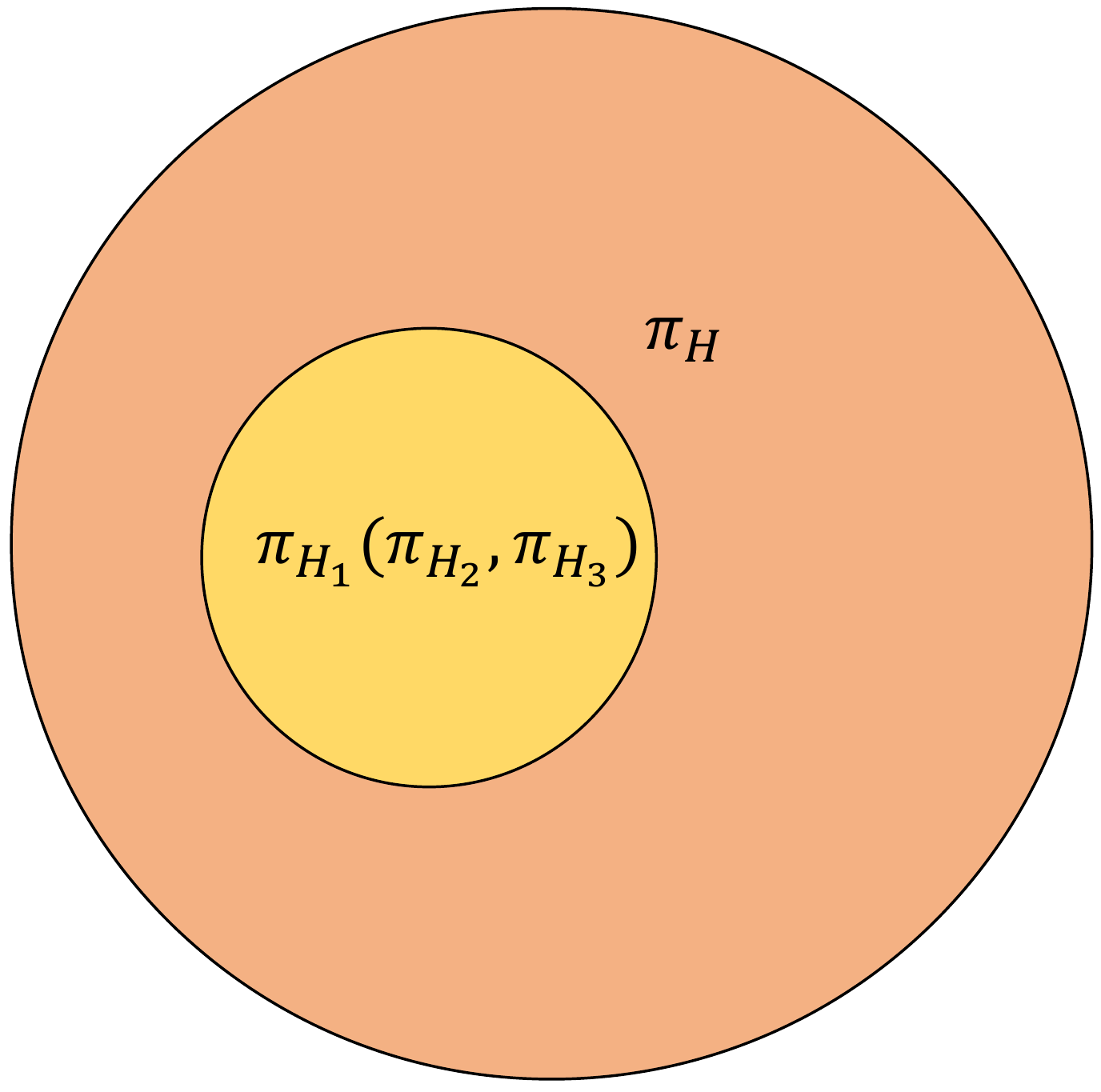}}\hfill
\subfloat[$\epsilon \rightarrow \infty$, $\Pi_{adv} = \Pi$]{\includegraphics[width = .45\linewidth]{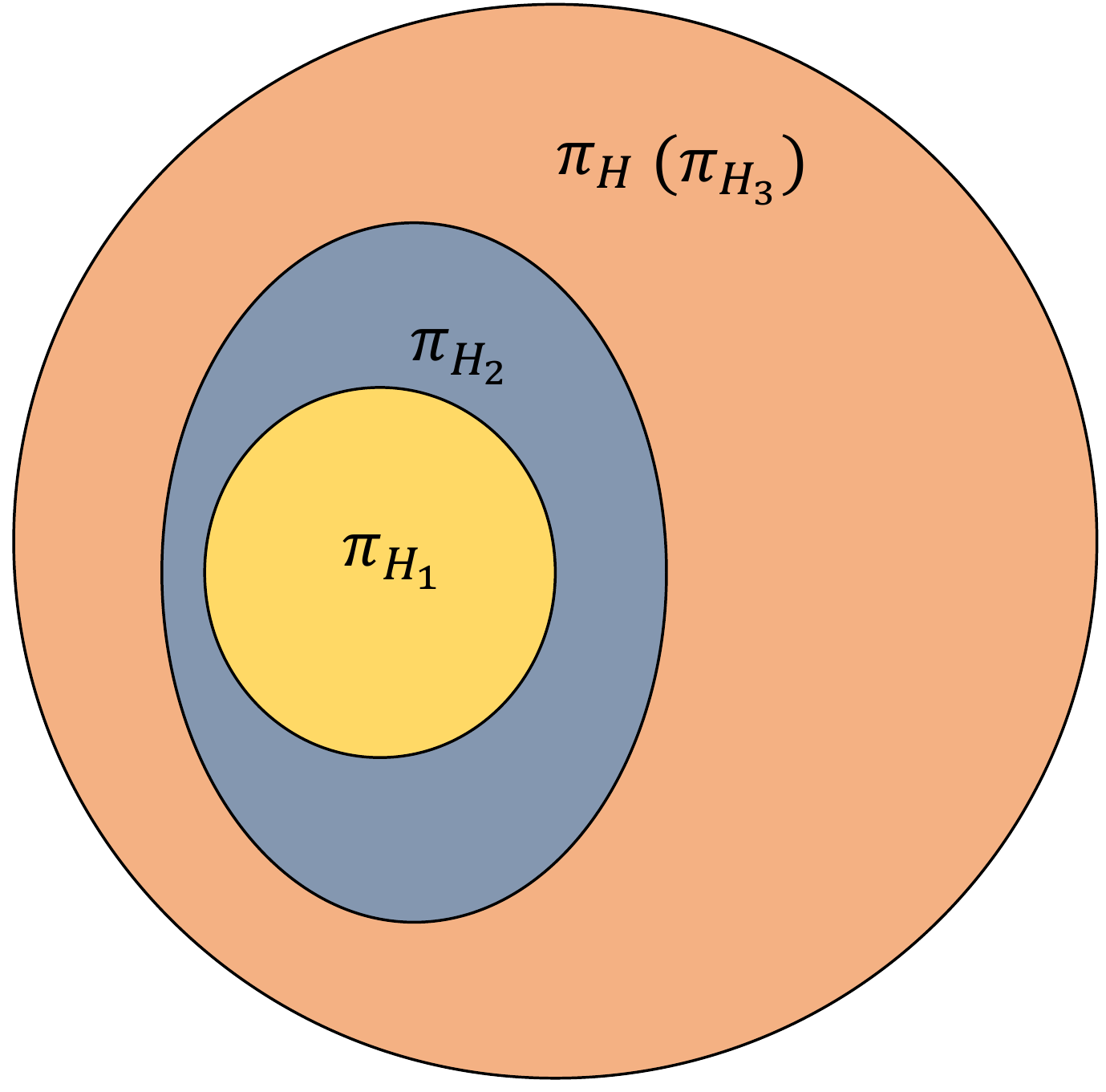}}

\caption{The relationship among the policy sets $\pi_{\mathcal{H}}$, $\pi_{\mathcal{H}_1}$, $\pi_{\mathcal{H}_2}$ and $\pi_{\mathcal{H}_3}$ with different noise level $\epsilon$.  } 
\label{fig:relation of set}
\end{figure}

Therefore, we solve where-to-find problem by finding the optimal solution $h^*$ in $\mathcal{H}_3$, and next we will show how to solve it. 

\subsubsection{Assumptions on the Adversary in the Function Space $\mathcal{H}_3$} 
 After careful examinations on the present targeted attacks in the function space of $H_3$, we find that the existing methods~\cite{8,19,20} implicitly follow a ``pessimistic'' assumption as
\begin{assumption}
\label{assumption: pessimistic}
For the adversary $h \in \mathcal{H}_3$ and the victim policy $\pi$, \textbf{the pessimistic assumption} states: 
\begin{equation}
    Q^{\pi_h}(s,a) \approx Q^{\pi}(s,a),
\end{equation}
where $Q^{\pi_h}(s,a)$ is the the Q-value function of the attacked policy $\pi_h$.  
\end{assumption}
We call this assumption as the pessimistic assumption because it assumes that the performance of the attacked policy is close to the the victim policy. In other words, we pessimistically assume the effect of the adversary. The objective function of existing approaches~\cite{19} is exactly the same as the elemenatary function in Eq.~\eqref{origin} by substituting $Q^{\pi_h}$ with $Q^\pi$; and the objective function in~\cite{8} can be obtained by substituted $Q^{\pi_h}$ into the elemenatary function in Eq.~\eqref{origin} with a robust Q-values function $Q^{RS}$.
 
However, the pessimistic assumption may be inappropriate since practically, the performance of the adversarial attack~\cite{8,19} is close to the worst policy $\pi^-$ rather than the victim policy $\pi$. Therefore, we present another reasonable ``optimistic'' assumption as
\begin{assumption}
For the adversary $h \in \mathcal{H}_3$ and $\pi^-$ is the worst policy, \textbf{the optimistic assumption} states 
\begin{equation}
\label{equ_optms}
    Q^{\pi_h}(s,a) \approx Q^{\pi^-}(s,a),
\end{equation}
where $ Q^{\pi_h}$ and $Q^{\pi^-}$ are Q-value functions for the attacked policy $\pi_h$ and the worst policy $\pi^-$, respectively. 
\end{assumption} 
 We call this assumption as the optimistic assumption because we assume that the performance of the attacked policy is close to the performance of the worst policy, which means we are optimistic to the effect of the adversary. We follow this optimistic assumption and further design a two-stage optimization method to solve this problem in the following, which is the first attempt following this assumption to the best of our knowledge.

\section{Methodology}
\label{section: Metho}

\begin{algorithm}[t]
    \caption{Two-stage Attack}  
    \label{pseudocode}
    \begin{algorithmic}[1]
        \REQUIRE victim policy $\pi$
        \STATE \textbf{Stage 1}: Learn the deceptive policy $\pi^{-}$ by using RL algorithm (like PPO) to solve
        \begin{equation}
        \pi^{-}\in\mathop{\arg\min}_{\pi\in\Pi} R(\pi).
        \end{equation}
        \STATE \textbf{Stage 2}: When observe current state $s$, calculate the adversarial noise via the deceptive policy:
        \begin{equation}
            \hat{h}^*(s) = \mathop{\arg\min}_{h\in\mathcal{H}_3} D_{\text{KL}} (\pi_{h}(\cdot|s) \| \pi^-(\cdot|s)),
        \end{equation}
        and apply it to current observation.
    \end{algorithmic}  
\end{algorithm}

Following previous discussions, we now present an adversarial attack to obtain the adversarial policy in function space $\mathcal{H}_3$ under the optimistic assumption and the pseudo code is outlined in Algorithm~\ref{pseudocode}.

\subsection{Two-stage Optimization}
\label{s2}

The original problem~(\ref{origin}) is intractable to solve since the attacker is required to infer the environmental dynamics and the exploration mechanism in
DRL, which inevitably leads to a shift in the distribution of the state. To address this issue, we first introduce a deceptive policy to explore the worst case in the environment that can minimize the accumulated reward. Moreover, we denote $\Pi_{d}$ as the set of deceptive policies, which is formally defined as
\begin{myDef}[\textbf{The deceptive policy set of} \bm{$\Pi_d$}]
$\Pi_d$ is a set of deceptive policies which can minimize the accumulated reward on an MDP as 
\begin{equation}
   \Pi_d = \left\{ \pi|\pi = \mathop{\arg\min}_{\pi ~ \in \Pi} R(\pi) \right\}. \nonumber
\end{equation}
\end{myDef}

Theoretically, it is noted that if the adversary can interact with the environment then the adversary can learn a deceptive policy, i.e. $\Pi_d \cap \Pi_{adv} \neq \varnothing$. Moreover, the deceptive policy can also be specified if the adversary has some expert knowledge which can help the adversary to minimize the victim's reward.

Following the optimistic assumption, we use Q-values $Q^{\pi^-}(s,a)$ of another deceptive policy $\pi^- \in \Pi_d \cap \Pi_{adv}$ to estimate $Q^{\pi_h}(s,a)$. As policy $\pi^- \in \Pi_d$ such that $\pi^- = \mathop{\arg\min}_{\pi \in \Pi}\expect_{a \sim \pi}\left[ Q^\pi(s,a) \right]$, an optimal solution of $\pi_h = \mathop{\arg\min}_{\pi \in \Pi}\expect_{a \sim \pi}[Q^{\pi^-}(s,a)]$ is $\pi_h = \pi^-$. Therefore, we reformulate problem~(\ref{origin}) as minimizing the distance between $\pi_h$ and $\pi^-$ on a trajectory determined by policy $\pi^-$:
\begin{equation}
\begin{aligned}
\label{indirect-1}
   h^* &= \mathop{\arg \min}_{h \in \mathcal{H}_3} \expect_{s \sim d^{\pi^-}} \left[ D_{TV}(\pi_h(\cdot|s)\| \pi^-(\cdot|s)) \right], \ \\
\end{aligned}
\end{equation}

where $\pi^- \in \Pi_d \cap \Pi_{adv}$ is the targeted policy to explore the worst case in the MDPs. Here, $D_{TV}(\cdot\|\cdot)$ is the total variation (TV) distance between two policy distributions. Specifically, $d^{\pi}$ is the distribution for future state under the policy as   $d^{\pi}(s)=(1-\gamma)\sum_{t=0}^{\infty}\gamma^tP(s_t = s|\pi)$. Since the deceptive policy set $\Pi_d$ is a subset of $\Pi$, we can reduce the search space by only considering the policies trained with flipped reward for the victim, yielding a more efficient optimization. 

Considering that the TV distance does not admit a closed-form expression for the statistical mixture distributions (e.g., GMM models)~\cite{24}, it often is necessary to conduct Monte Carlo approximations or a numerical integration. Nevertheless, these operations do not guarantee the deterministic lower and upper bounds. To address this issue, we propose to use KL-divergence which can be used to upper bound the TV distance~\cite{25} and is widely used in the constraint for trust regions in RL algorithms~\cite{26}. In particular, we introduce a new objective with KL-divergence instead of problem~(\ref{indirect-1}) as 
\begin{equation}
\label{indirect-2}
   h^\ast = \mathop{\arg \min}_{h \in \mathcal{H}_3} \expect_{s \sim d^{\pi^-}} \left[D_{KL}(\pi_h(\cdot|s) \| \pi^-(\cdot|s))\right].
\end{equation}

The problem is also intractable because the state-distribution is generated by $\pi_h$ which is in a complex function space. Therefore, we reformulate it as: 

\noindent \textbf{A two-stage optimization problem} by inferring the deceptive policy as an intermediate step: 
\begin{equation}
\begin{aligned}
\label{indirect-3}
     \hat{h}^\ast(s) &= \mathop{\arg \min}_{h \in \mathcal{H}_3} D_{KL}(\pi_h(\cdot|s) \| \pi^-(\cdot|s)),~\forall s \in \mathcal{S}, \\ 
\end{aligned}
\end{equation}
where $\pi^- \in \Pi_d \cap \Pi_{adv}$ is the intermediate deceptive policy. Afterwards, we directly solve problem~\eqref{indirect-3} in function space $\mathcal{H}_3$ by defining the target policy as $\pi^{-}$.


\noindent \textbf{Solving the two-stage optimization problem}
\label{s3} Note that when the policy $\pi^-$ is given, problem~(\ref{indirect-3}) can be seen as a typical targeted attack in supervised learning. The adversary independently adds perturbation on each state $s$ and treats $\pi^-(\cdot|s)$ as a target function. This is different from previous work~\cite{6}, which is an untargeted attack.
In this paper, we propose to use the Projected Gradient Descent (PGD)~\cite{13}, which is a popular method for targeted attacks. The adversary iteratively updates the observation as
\begin{equation}
    s_{k+1} = s_{k} - \eta, \nonumber
\end{equation}
where $k$ is the number of iterations,
$s_0 = s$ is the original observation, and $\eta$ is the perturbation. The perturbation is based on the type of norm. For the commonly used $l_2$-norm, the perturbation can be computed by the projected gradient descent (PGD) on the negative loss of Eq.~\eqref{indirect-3} as 
\begin{equation}
    \eta  = \frac{ \epsilon'\sqrt{\left|s\right|} \nabla_s \Big(D_{KL}\big(\pi(\cdot|s_k) \| \pi^{-}(s)\big)\Big)}{\left|\left|\nabla_s \big(D_{KL}(\pi(\cdot|s_k) \| \pi^{-}(\cdot|s) ) \big) \right|\right|_2}, \nonumber
\end{equation} 
where $\epsilon'$ is the step size to control the distance between the resultant and the original observations.

\subsection{Theoretical Analysis}
\label{s4}




We now present a theoretical analysis on
 practical setting that provides a performance bound for our algorithm. 
First, we define the $\alpha$-deceptive policy set $\Pi^-(\alpha)$ to express a class of policies, after which we use this set to measure the adversary's capability. 
\begin{myDef}[\textbf{The \bm{$\alpha$}-deceptive policy set} \bm{$\Pi^-(\alpha)$}]
An $\alpha$-deceptive policy set $\Pi^-(\alpha)$ corresponds to the policies that are $\alpha$-close to the performance of $\pi^- \in \Pi_{d}$ as 
\begin{equation}
   \Pi^-(\alpha) = \{\pi|R(\pi) \leq R(\pi^-) + \alpha, \pi \in \Pi \}. \nonumber
\end{equation}
\end{myDef}
With the definition of $\alpha$-deceptive policy set $\Pi^-(\alpha)$, we further define the ability of an adversary with parameter $\hat{\alpha}$.
\begin{myDef}[\textbf{The \bm{$\hat{\alpha}$}-adversary}]
     Let $\Pi_{adv}$ be a policy set that is accessible to the adversary, an adversary is an $\hat{\alpha}$-adversary if 
    \begin{equation}
        \hat{\alpha} = \inf \{\alpha | \Pi_{adv} \cap \Pi^-(\alpha) \neq \varnothing \}. \nonumber
    \end{equation}
\end{myDef}

It is noted that the adversary is a $0$-adversary when it can minimize the reward (or reach the worst case state) by exploring the environments or obtaining the environmental dynamics from an expert. With the definition of the adversary's capability $\hat{\alpha}$, we further provide an upper bound of the performance after an attack by an  $\hat{\alpha}$-adversary.

\begin{table}[!t]
\footnotesize
\tabcolsep 14pt 
\begin{tabular*}{\textwidth}{c|cccc}
\toprule
Adversary   & Ant & Hopper & Walker2d & HalfCheetah\\    
\hline
Noise level of $\epsilon$ &0.15 & 0.07&  0.15& 0.15 \\
\hline
No Noise          &5861.10 $\pm$ 609.63  & 3290.41 $\pm$  397.13  & 4491.23 $\pm$ 674.98 & 7102.41 $\pm$ 121.03\\

Random &5528.39 $\pm$ 609.63    &2850.41 $\pm$ 797.75 & 4275.17 $\pm$ 871.11& 5336.31 $\pm$ 1574.15\\

Critic &5071.47  $\pm$   970.93  &2090.47 $\pm$1025.50  &4131.07 $\pm$ 639.36 &5584.03 $\pm$ 879.24 \\

Stochastic MAD &1648.23 $\pm$ 820.80   & 1868.66 $\pm$ 686.15 &901.84 $\pm$ 475.67 & 1798.58 $\pm$ 1014.83\\

RS &412.87 $\pm$  247.87 & 2808.47$\pm$ 871.20 & 1299.01 $\pm$ 307.40 & 338.81 $\pm$ 526.71\\
Snooping &3978.19 $\pm$ 658.76 &1832.97 $\pm$ 562.45 &1265.88 $\pm$ 782.60& 1967.98 $\pm$ 1077.90\\
${\rm ``Optimal"\ Attack}$    &  -493.22 $\pm$ 40.49    & 637.30 $\pm$ 3.32 &879.85 $\pm$ 36.14   & \textbf{-657.60 $\pm$  288.10}  \\
${\rm Our\ two\text{-}stage}$  & \textbf{-841.01 $\pm$ 494.88} & \textbf{112.49 $\pm$ 148.64} & \textbf{-29.93 $\pm$ 7.35} & -314.65 $\pm$ 54.75\\
\bottomrule
\end{tabular*}

\caption{The average reward of the victim policy (PPO) under adversarial attack on MuJoCo. We \textbf{bold} the best attack reward (the lowest attacked policy's reward) over all attacks.}
\label{table-4}
\end{table}

\begin{lemma}[\textbf{Upper bound of the \bm{$\hat{\alpha}$}-adversary's performance}]
\label{lemma-main}
Let the adversary be an $\hat{\alpha}$-adversary, the performance of the perturbed policy $\pi_h$ satisfies:
\begin{equation}
\label{inequal: all}
    R(\pi_h) \leq \hat{\alpha} + \frac{C\beta_1}{1-\gamma} + \frac{2\gamma C\sqrt{\beta_0/2}}{(1-\gamma)^2} + R(\pi^-), \nonumber
\end{equation}
where $\beta_0 = \max_{s \in S}\left\|D_{KL}(\pi_h(\cdot|s) \| \pi^-(\cdot|s))\right\| $, $C = \max_s \left| \expect_{a \sim \pi_h}\left[A^{\pi^-}(s,a)\right] \right|$ and $\beta_1 = \max_{s, a} \|\frac{\pi_h(a|s)}{\pi^-(a|s)} - 1\|$, A is the advantage function.
\end{lemma}

The complete derivation is provided in Appendix A. Lemma \ref{lemma-main} implies that the performance of the adversarial attack decreases as the adversary's capability $\hat{\alpha}$ and the distance between policy $\pi_h$ and $\pi^-$ decrease. 
We further show that under an appropriate noise level, our attacked method can be stronger than the existing sub-optimal adversary, as summarized in Theorem \ref{thm-4}:

\begin{thm}[\textbf{\bm{$\hat{\alpha}$}-adversary is stronger than other adversary under some conditions}]
\label{thm-4}
Let $e$ be an arbitrary adversarial attack algorithm, and we set $\alpha_{e} = R(\pi_e) - R(\pi^-)$ and $\beta_1 = \max_{s, a} \|\frac{\pi_h(a|s)}{\pi^-(a|s)} - 1\|$. If $\beta_1$ satisfies:
\begin{equation}
     \beta_1 < \frac{-\sqrt{2}\gamma C + \sqrt{2\gamma^2C^2+4(\alpha_e - \hat{\alpha})(1-\gamma)^3}}{2(1-\gamma)C}, \nonumber
\end{equation}
then the performance of the victim policy after perturbed by our algorithm satisfies $R(\pi_h) < R(\pi_{e})$. Here $C$ follows the definition in Lemma~\ref{lemma-main}.
\end{thm}
The proof of Theorem~\ref{thm-4} is provided in Appendix A. Based on this analysis, we reduced $\beta_1$ by targeted attacks in our experiments and obtained promising results in both Atari and MuJoCo environments.

\section{Experiments}
In this section, we conduct extensive experiments to evaluate and show the effectiveness of our method. 

\begin{table}[!t]
\footnotesize
\tabcolsep 14pt 
\begin{tabular*}{\textwidth}{c|c|cccc}
\toprule


Adversary   & Optimizer & Qbert & Pong & SpaceInvaders   & BeamRider   \\    
\hline
Noise level of $\epsilon$ & N/A& 0.007& 0.004 &  0.012 & 0.014\\
\hline
No Noise        & N/A &9402.1$\pm$ 1719.2  & 20.5 $\pm$  0.6   & 509.0  $\pm$ 135.3 & 2251.5 $\pm$ 198.6\\
\hline
Random & N/A & 7380.0 $\pm$ 542.8& 19.9 $\pm$ 1.2 &216.8 $\pm$ 61.9 &  976.2 $\pm$ 440.3\\
\hline
Critic &FGSM& 192.5 $\pm$110.1 & -19.0 $\pm$ 2.8 &101.0 $\pm$ 46.3 & 134.2 $\pm$ 81.5\\
&PGD(10)& 157.5 $\pm$ 60.5& -19.7 $\pm$ 1.6 & 96.0 $\pm$ 86.9 & 149.6 $\pm$ 76.0 \\
\hline
${\rm Deterministic}$ &FGSM& 275.0 $\pm$ 169.4 & -20.6$\pm$0.3 & 79.5 $\pm$ 33.7 & 189.2 $\pm$ 59.8\\
${\rm MAD}$ &PGD(10)& 210.0 $\pm$ 206.3 & \textbf{-21.0$\pm$0.0}& 171.8 $\pm$ 23.1& 211.2 $\pm$ 33.5\\
\hline
${\rm Stochastic}$  &FGSM& 246.3 $\pm$ 82.3 & -18.3 $\pm$ 3.3&111.3 $\pm$ 59.9 & 451.0 $\pm$ 118.0\\
${\rm MAD}$  &PGD(10)& 237.5 $\pm$ 110.1 & -20.9$\pm$0.1 & 189.5 $\pm$ 33.4 & 525.8 $\pm$ 83.2\\
\hline
${\rm Our\ two\text{-}stage}$   &FGSM& 193.8$\pm$ 85.1 &  \textbf{-21.0$\pm$0.0}  &  47.0 $\pm$ 26.5 &  308.0 $\pm$ 121.1\\
  &PGD(10)&\textbf{13.8 $\pm$ 13.9}  &  \textbf{-21.0$\pm$0.0}  & \textbf{0.0$\pm$0.0} &  \textbf{6.6 $\pm$ 7.3} \\
\bottomrule
\end{tabular*}
\caption{The average reward of the victim policy (DQN) under adversarial attack on Atari. We \textbf{bold} the best attack reward (the lowest attacked policy's reward) over all attacks.}
\label{table-1}
\end{table}

\subsection{Experimental Setup}

\noindent \textbf{Implementation Details for MuJoCo.} 
First, we evaluate the effectiveness of our adversarial attacks on four OpenAI Gym MuJoCo continuous environments -- {\it Ant}, {\it Hopper}, {\it Walker2d} and {\it HalfCheetah}, by using the implementation and the pretrained models in \cite{9}.

\noindent \textbf{Implementation Details for Atari.} 
Further, we conduct experiments to evaluate our adversarial attacks on four Atari games --- {\it Pong}, {\it BeamRider}, {\it Qbert} and {\it SpaceInvaders}. 
The victim policies are mainly trained by three classic reinforcement learning algorithms --- DQN~\cite{1, 27}, A2C~\cite{28} and PPO~\cite{29}. We use DQN and A2C implementation in~\cite{30} and PPO implementation in~\cite{31}. 
All these policies achieve competitive performance in the Atari environments.

\noindent \textbf{Details of Deceptive Policies.}
We train five deceptive policies to evaluate our two-stage adversarial attacks. 
For PPO victims, we train deceptive policy by PPO; while for A2C and DQN victims, we train deceptive policies by A2C. To reduce the variance of our two-stage attack performance, we evaluate our two-stage attacks by choosing the policy with median reward in these five deceptive policies. All these deceptive policies are trained with the same hyper-parameters as their victims. 

\noindent \textbf{Additional Details}
In MuJoCo environments, we use the released victim policies in~\cite{9}. In Atari environments, the victim policy is trained with GeForce GTX 1080Ti 11G and Intel(R) Xeon(R) CPU E5-2680 v4 @ 2.40GHz. We train 10M steps for DQN, which needs fewer than three days each. We train 40M steps for A2C, which needs no more than one day each. We train 10M steps for PPO, which needs no more than 1 day each. The deceptive policy is trained for 1M steps each in Atari environments and 1,000 steps each in MuJoCo environments---significantly fewer than the number of steps required for the victim policy. It takes less than a minute to generate the adversarial observation in all Atari environments by two-stage attack~(FGSM) and less than two minutes to generate the adversarial observation in Pong, SpaceInvaders and Qbert, as well as less than 10 minutes in BeamRider by two-stage attack~(PGD). In MuJoCo environments, it takes less than a minute for two-stage~(PGD) for each trajectory. It is worth noting that we use the KL-divergence between the attacked policy and the deceptive policy plus the entropy of the attacked policy
to generate a smoother attacked policy.

\noindent \textbf{Evaluation on MuJoCo Environments.}
As a comparison, we evaluate our methods against the SOTA representative works following~\cite{9}, 
including Random attacks, the \emph{untargeted attack} of 
Stochastic MAD attack~\cite{8} and Snooping attack~\cite{29}, along with the \emph{targeted attack}
of Critic attack~\cite{19} and Robust Sarsa (RS) attack~\cite{8}. As Deterministic MAD is not applicable to the environment of continuous action space, we follow previous work~\cite{8} and just evaluate Stochastic MAD here. We also compare our method with the SOTA \emph{learning-based attacks} of ``Optimal'' attack~\cite{9}. 
For fair comparisons, we choose the similar setup in~\cite{9}. In particular, we consider the perturbation in $l_\infty$-norm, and run the agents without attacks, as well as under attacks for 50 episodes, and report the mean and standard deviation of episode reward. 

\noindent \textbf{Evaluation on Atari Environments.}
As learning based methods are inefficient in Atari environment, we evaluate the vulnerability of victim policy with Random attack and three optimization-based attacks, including the \emph{untargeted attacks} of Deterministic MAD attack~\cite{6} and Stochastic MAD  attack~\cite{8}, as well as the \emph{targeted attacks} of Critic attack~\cite{19}.
For fair comparisons, we consider the similar setup as the previous work. In particular, we use perturbation of $l_2$-norm as in~\cite{6,19} and ~\cite{9}, and run the agents without attacks---as well as under attacks for five episodes during testing.

\begin{table}[!t]
\footnotesize
\tabcolsep 13pt 
\begin{tabular*}{\textwidth}{c|c|cccc}
\toprule
Adversary   & Optimizer & Qbert & Pong & SpaceInvaders  & BeamRider   \\ 
\hline
Noise level of $\epsilon$ &N/A & 0.006 &  0.008 & 0.014 & 0.014 \\
\hline
No Noise     &  N/A& 13564.0 $\pm$ 1716.7  & 21.0 $\pm$ 0.0    & 634.2 $\pm$ 148.3 &  15222.1 $\pm$ 2797.3 \\
\hline
Random & N/A&14394.0 $\pm$ 1085.6 & 21.0 $\pm$ 0.0 & 459.2 $\pm$ 129.0 & 14938.8 $ \pm$ 1813.5\\
\hline
Critic &FGSM& 168.0 $\pm$ 108.7 &\textbf{-21.0 $\pm$ 0.0 }&150.0 $\pm$ 79.2 & 207.7 $\pm$ 76.2\\\
&PGD(10)& 437.0 $\pm$ 103.6& \textbf{-21.0 $\pm$ 0.0 } &148.4 $\pm$ 97.8& 248.2 $\pm$ 48.3\\
\hline
${\rm Deterministic}$ &FGSM& 193.0 $\pm$ 132.3 & \textbf{-21.0 $\pm$ 0.0 } &109.4 $\pm$ 25.0&  279.8 $\pm$ 25.7\\
${\rm MAD}$&PGD(10)& 412.0 $\pm$ 168.9 & -21.0 $\pm$ 0.1 &140.4 $\pm$ 76.3& 416.3 $\pm$ 102.2\\
\hline
${\rm Stochastic}$  &FGSM& 156.0 $\pm$ 75.6 & \textbf{-21.0 $\pm$ 0.0 } &183.0 $\pm$98.6 & 529.3 $\pm$ 75.5\\
${\rm MAD}$ &PGD(10)& 341.0 $\pm$ 107.4 & -20.8 $\pm$ 0.3 &132.8 $\pm$  70.8 & 586.4 $\pm$ 62.4\\
\hline
${\rm Our\ two\text{-}stage}$   &FGSM& \textbf{47.0 $\pm$ 39.7}  &  \textbf{-21.0 $\pm$ 0.0 }  & 80.4 $\pm$ 60.6 &  455.4 $\pm$ 130.3 \\
&PGD(10)&  301.0 $\pm$ 88.4  &  \textbf{-21.0 $\pm$ 0.0 }  & \textbf{26.6 $\pm$34.4}  &  \textbf{188.3 $\pm$ 127.0} \\
\bottomrule
\end{tabular*}
\caption{The average reward of the victim policy (A2C) under adversarial attack on Atari. We \textbf{bold} the best attack reward (the lowest attacked policy's reward) over all attacks.}
\label{table-3}
\end{table}

\begin{table}[!t]
\footnotesize
\tabcolsep 13pt 
 \begin{tabular*}{\textwidth}{c|c|cccc}
 \toprule
 Adversary   & Optimizer & Qbert & Pong & SpaceInvaders   & BeamRider   \\    
 \hline
 Noise level of $\epsilon$ & N/A   &0.002    &  0.001  & 0.014 & 0.008 \\
 \hline
 No noise      & N/A &16999.0 $\pm$ 2008.7 & 20.6 $\pm$ 0.3 & 952.2 $\pm$ 229.4 & 1873.4 $\pm$ 771.1 \\
 \hline
 Random & N/A & 16275.0 $\pm$ 1002.5 & 20.7 $\pm$ 0.4 & 768.6 $\pm$ 186.3 & 1788.2 $\pm$ 368.7\\
 \hline
 ${\rm Stochastic }$  &FGSM& 331.0 $\pm$ 122.3 & -18.5 $\pm$ 0.6 & 209.4 $\pm$ 135.5 & 526.0 $\pm$ 43.8\\
 ${\rm MAD}$&PGD(10)& 340.0 $\pm$ 130.8 & -19.0 $\pm$ 1.0 & 172.4 $\pm$ 31.6 & 508.0 $\pm$ 74.0\\
 \hline
 ${\rm {Our~two{-}stage}}$   &FGSM& \textbf{127.0 $\pm$ 90.3} &  \textbf{-21.0$\pm$0.0}  & 173.0 $\pm$ 34.9 &  \textbf{364.5 $\pm$ 153.7}\\
    &PGD(10)& 271.0  $\pm$  67.0&  \textbf{-21.0$\pm$0.0}  & \textbf{166.0 $\pm$ 38.4} &  371.4$\pm$ 162.6\\
 \bottomrule
 \end{tabular*}
 \caption{The average reward of the victim policy (PPO) under adversarial attack on Atari. In each line we \textbf{bold} the best attack reward (the lowest attacked policy's reward) over all attacks.}
 \label{table-2}
 \end{table}

\begin{figure}[!t]
    \centering
    \includegraphics[scale=0.25]{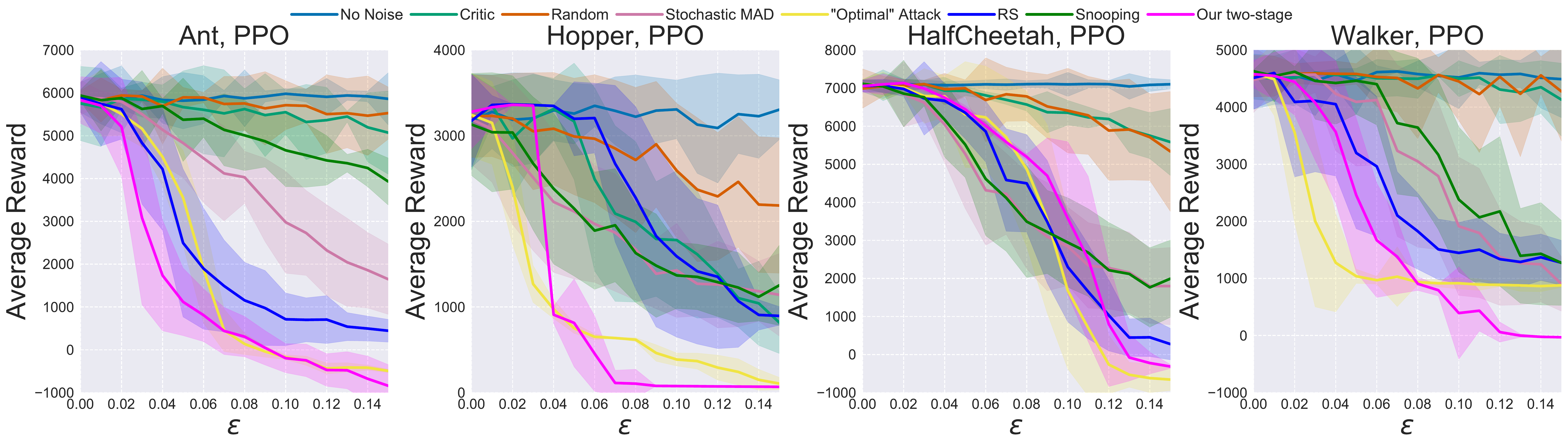}
    \caption{Attacking PPO agents in MuJoCo environments under different $\epsilon$ noise level. Each color in this figure represents the results of 50 episodes. The shadowed area is the variance of the return of episodes. In the remaining part, we omit the description of the shadowed area as they all represent the variance of the return of episodes.}
    \label{fig:mujoco}
\end{figure}

\subsection{Experimental Results}
In this experiment, we study the vulnerability of victim policy to our proposed two-stage attack. We generate adversarial perturbations with deceptive policy by two common adversarial optimizers: the Fast Gradient Sign Method (FGSM)~\cite{12} as well as the Projected Gradient Descent (PGD)~\cite{13}. Specifically, we generate adversarial perturbations with deceptive policy by PGD with 10 iterations.

\noindent \textbf{Results in MuJoCo Environments.}
Table~\ref{table-4} presents the results on attacking PPO agents in MuJoCo environments. In most tasks, our two-stage approach markedly outperforms all other baselines. Firstly, as a targeted attack, our two-stage attack achieves a clearly lower reward than two untargeted attacks of the Stochastic MAD attack and the snooping attack. It indicates that targeted attack is much stronger than the untargeted attacks in optimization-based algorithms. Besides, compared to the targeted attacks of the critic attack and the RS attack follows the pessimistic assumption, our two-stage attack under the optimistic assumption can significantly reduce the attacked policy's reward. It indicates that we can find a more effective adversary with the optimistic assumption. Additionally, our optimization-based two-stage attack outperforms the learning-based ``optimal'' attack in most environments, which indicates we can design an effective attack by leveraging an efficient optimizer.

In Fig.\ref{fig:mujoco}, we show the lowest performance of the attacked victim policy achieves a lower reward than any other attacked policies. Additionally, our algorithm also misleads the agent to the opposite direction (reward $<$0) in Walker2d, which indicates our attack is stronger than all previous attacks.

\noindent \textbf{Results in Atari Environments.}
Table~\ref{table-1}, Table~\ref{table-3}, and Table~\ref{table-2} present the results on attacking DQN agents, A2C agents, and PPO agents, which show our two-stage attack significantly outperforms the alternative methods in the Atari environments. 
In particular, for the Atari environments {\it SpaceInvaders} and {\it BeamRider}, as a targeted attack, our two-stage attack achieves significantly lower reward than two untargeted attacks of the Stochastic MAD attack and the Deterministic MAD attack.
These results show that it is a better choice to use the targeted attack algorithm to generate the perturbation. 
Also, despite of the fact that the critic attack and our two-stage attack all belongs to targeted attacks, our two-stage attack under the optimistic assumption significantly outperforms the critic attack with the same optimizer on attacking DRL agents in most environments.
It indicates that we can the adversary with our optimistic assumption than the adversary with the pessimistic assumption. Generally, the adversary with PGD is stronger than FGSM which indicates that a more powerful optimizer leads to a more powerful adversary.

In Fig.~\ref{fig:atari_dqn}, Fig.~\ref{fig:atari_a2c}, and Fig.~\ref{fig:atari_ppo}, we show the performance of the attacked victim policy trained by DQN, A2C and PPO is lower than other attacked policies under most of the noise level. It indicates that our approach is highly applicable to different noise level. Besides, We find that the lowest performance of the attacked policy under our two-stage attack is lower than all other adversaries in all settings. This result indicates the adversary follows the optimistic assumption is stronger than other adversaries.

Additionally, the deceptive policy is trained for 1M steps each in Atari environments and 1,000 steps each in MuJoCo environments---significantly fewer than the number of steps required for the victim policy (e.g., 40M steps for A2C and 10M steps for PPO in Atari).  Our optimization-based attacks generate the adversarial noise efficiently in Atari which is a high-dimensional space, which indicates the efficiency of our methods. On the contrary, the learning-based attacks are required to learning a high-dimensional mapping in MDP over the state and action space described (e.g., the dimension is $>$6000 in Atari), which is intractable to solve in general.

\noindent \textbf{Nearly lowest reward in Atari environments}  Our attack achieves nearly optimal performance on attacking DQN agents in Table~\ref{table-1}, which demonstrates the rationality of our optimistic assumption on the effectiveness of our adversary. In Fig.~\ref{fig:atari_best_dqn}, Fig.~\ref{fig:atari_best_a2c} and Fig.~\ref{fig:atari_best_ppo}, we show the rewards of the attacked policy under two-stage attack and the best performance under other attacks with different noise levels. Besides, we show the lowest reward in each task. We can see that our adversary achieves nearly the lowest reward when attacking DQN and A2C agents, and achieves competitive performance when attacking PPO agents. These results also demonstrate the rationality of our optimistic assumption.

\begin{figure}[!t]
    \centering
    \includegraphics[scale=0.24]{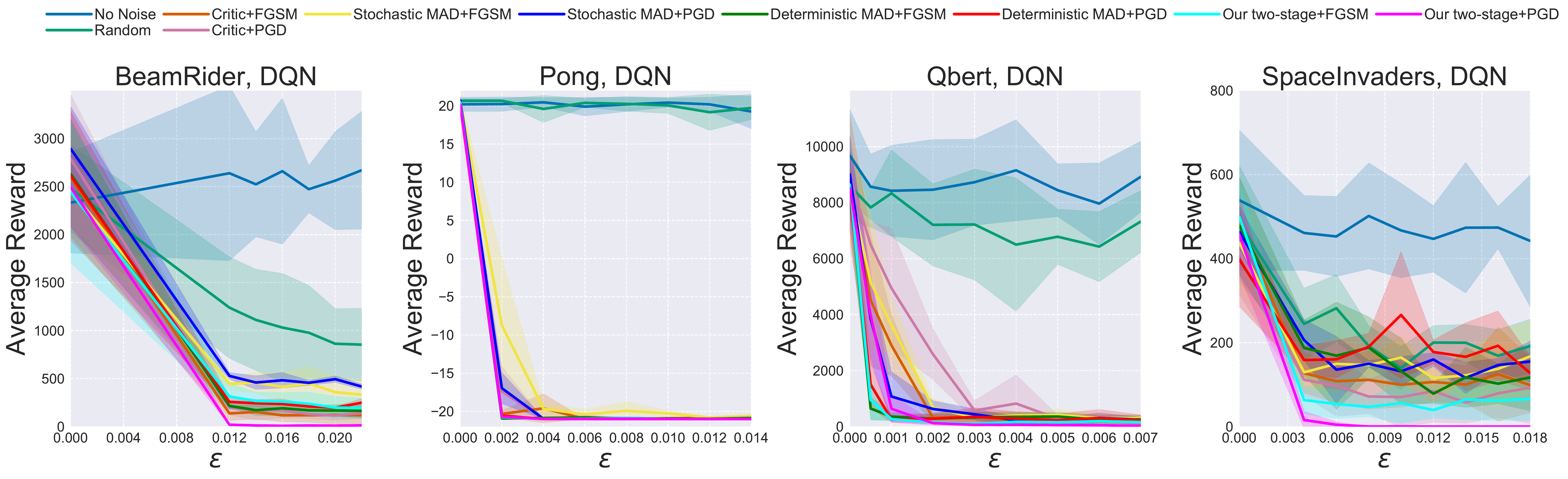}
    \caption{Attacking DQN agents in Atari environments under different $\epsilon$ noise level. Each color in this figure represents the results of 50 episodes.}
    \label{fig:atari_dqn}
\end{figure}

\begin{figure}[!t]
    \centering
    \includegraphics[scale=0.24]{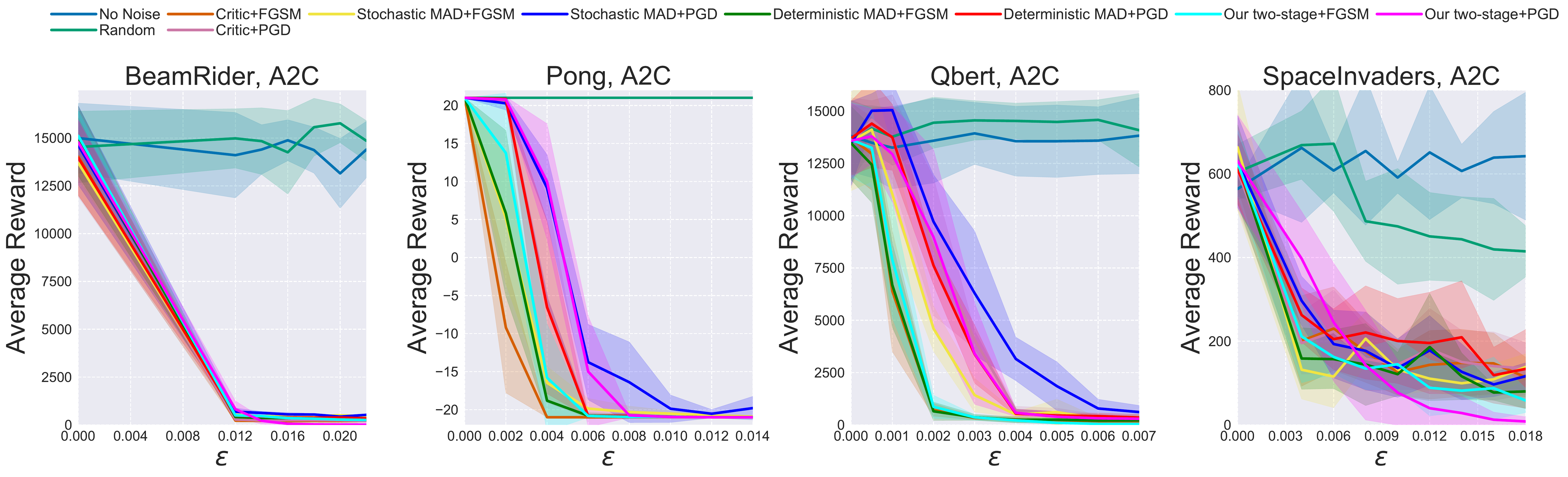}
    \caption{Attacking A2C agents in Atari environments under different $\epsilon$ noise level. Each color in this figure represents the results of 50 episodes.}
    \label{fig:atari_a2c}
\end{figure}

\begin{figure}[!t]
    \centering
    \includegraphics[scale=0.25]{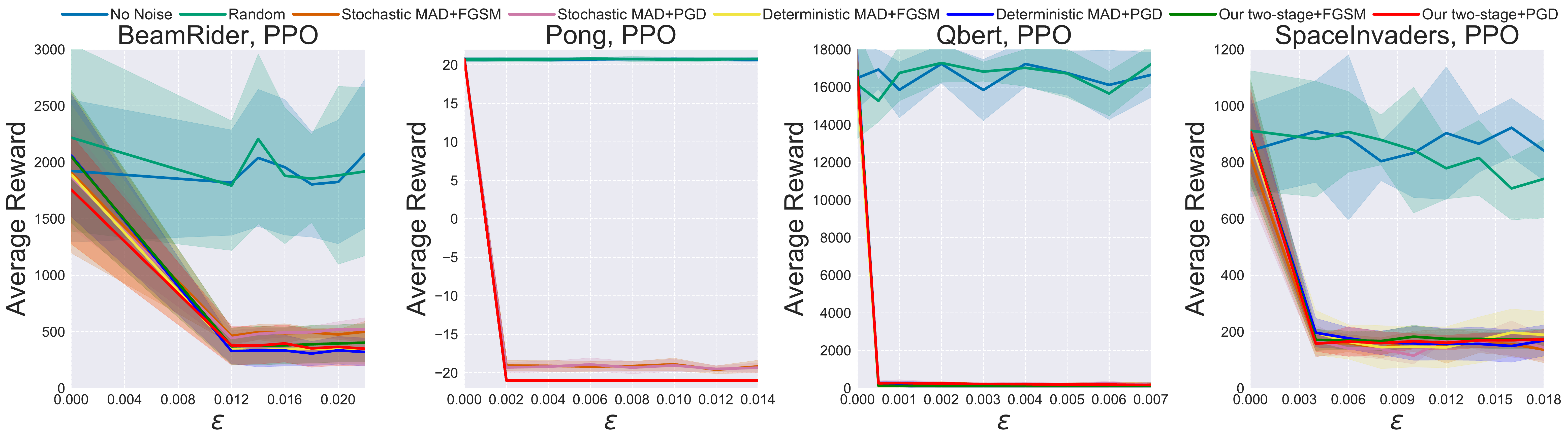}
    \caption{Attacking PPO agents in Atari environments under different $\epsilon$ noise level. Each color in this figure represents the results of 50 episodes.}
    \label{fig:atari_ppo}
\end{figure}

\begin{figure}[!t]
    \centering
    \includegraphics[scale=0.25]{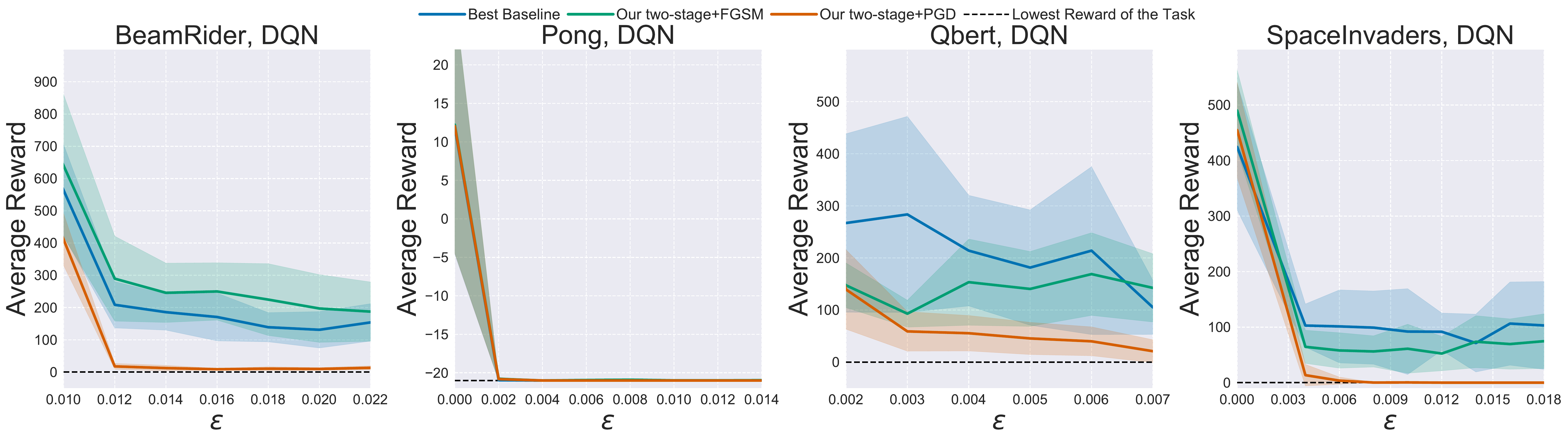}
    \caption{Attacking DQN agents in Atari environments under different $\epsilon$ noise level. The attacked policy under our Two-stage attack achieves the nearly lowest reward. Each color in this figure represents the results of 50 episodes.}
    \label{fig:atari_best_dqn}
\end{figure}

\begin{figure}[!t]
    \centering
    \includegraphics[scale=0.25]{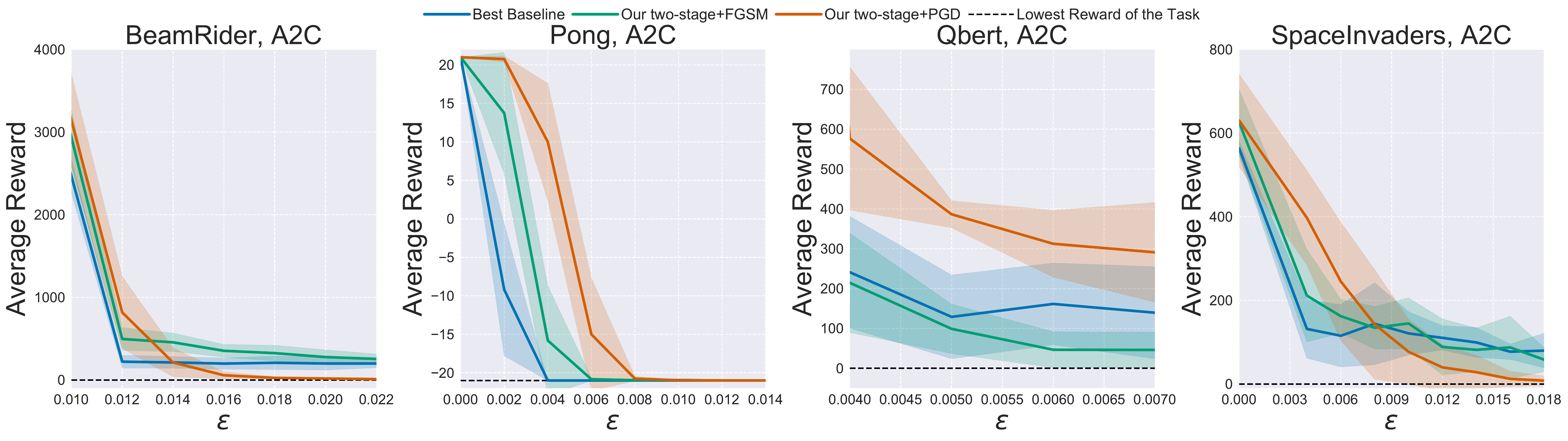}
    \caption{Attacking A2C agents in Atari environments under different $\epsilon$ noise level. The attacked policy under our two-stage attack achieves the nearly lowest reward. Each color in this figure represents the results of 50 episodes.}
    \label{fig:atari_best_a2c}
\end{figure}

\begin{figure}[!t]
    \centering
    \includegraphics[scale=0.25]{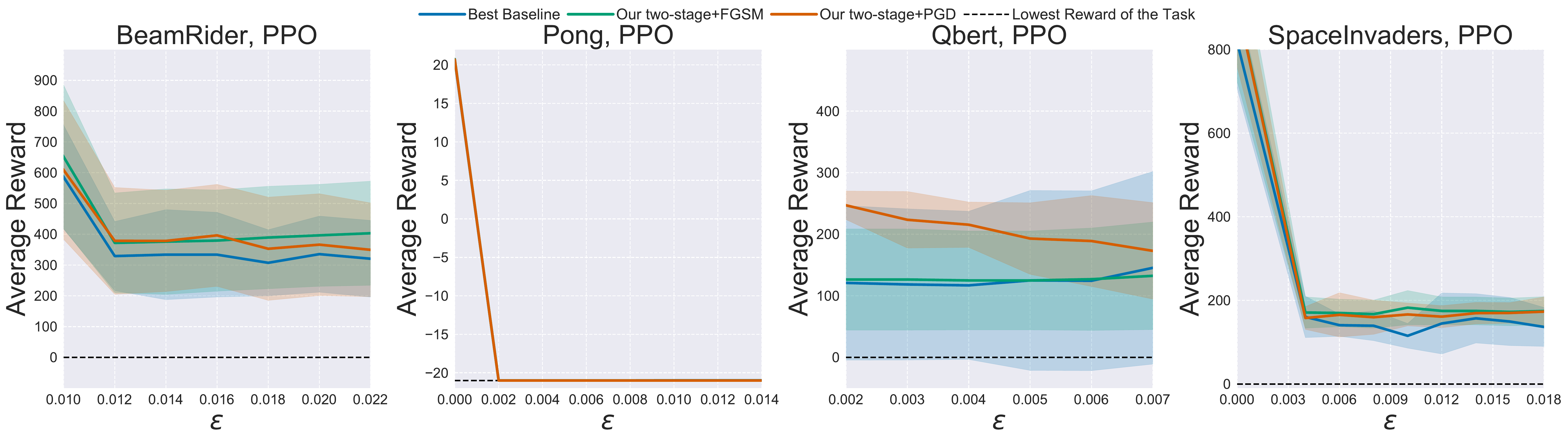}
    \caption{Attacking PPO agents in Atari environments under different $\epsilon$ noise level. Comparing to the other baselines, the reward of the attacked policy under our two-stage attack is closer to the lowest reward in each task. Each color in this figure represents the results of 50 episodes.}
    \label{fig:atari_best_ppo}
\end{figure}

\section{Conclusion}
We reformulate the SA-MDP in the function space and provides a framework to categorize and understand the existing optimization-based adversarial attacks on RL. We show that the adversary should be solved in the function space of targeted attacks following the optimistic assumption. Based on our understanding, we propose a two-stage method which can effectively and efficiently generate adversarial noise on the RL observation. Extensive experiments in both Atari and MuJoCo environments show the superiority of our method, which provides a possible way in assessing the adversarial robustness of reinforcement learning.

\Acknowledgements{This work was supported by the National Key Research and Development Program of China (No. 2020AAA 0104304, and 2017YFA0700904), NSFC Projects (Nos. 61620106010, 62061136001, 61621136008, 62076147, U19B2034, U1811461, U19A2081), Beijing NSF Project (No. JQ19016), Beijing Academy of Artificial Intelligence (BAAI),  Tsinghua-Huawei Joint Research Program, Tsinghua Institute for Guo Qiang, Tsinghua-OPPO Joint Research Center for Future Terminal Technology and Tsinghua-China Mobile Communications Group Co., Ltd. Joint Institute.}

\Supplements{Appendix A.}





\begin{appendix}

\section{Proof of Theorem 3}
In this part, we first begin with several lemmas and then provide a proof of Thm. 3.
With the notations in Sec. 4, the following lemma connects the difference in discounted total reward between two arbitrary policies to an expected divergence between them.
\begin{lemma} [\textbf{Upper bound for the performance gap between the attacked policy and the deceptive policy}]
\label{lemma-1}
Let $\beta = \expect_{s \sim d^{\pi^-}}\left[D_{TV}(\pi_h(\cdot|s)\| \pi^-(\cdot|s))\right]$, $C = \max_s \left| \expect_{a \sim \pi_h}\left[A^{\pi^-}(s,a)\right] \right|$ and $\beta_1 = \max_{s, a} \|\frac{\pi_h(a|s)}{\pi^-(a|s)} - 1\|.$ We have an upper bound on the performance gap between $\pi_h(s)$ and $\pi^-(s)$:
\begin{equation}
    R(\pi_h) - R(\pi^-) \leq \frac{C\beta_1}{1-\gamma} + \frac{2\gamma C\beta}{(1-\gamma)^2}. \nonumber
\end{equation}
\end{lemma}

\begin{proof}
Based on theorem 1 in \cite{1}, the performance of the attacked policy holds by the following bound:
\begin{equation}
\begin{aligned}
\label{lemma-1-2}
    R(\pi_h) - R(\pi^-)  \leq & \frac{1}{1-\gamma}\expect_{s \sim d^{\pi^-},~a \sim \pi_h}\left[A^{\pi^-}(s,a)\right] \\ +& 
    \frac{2\gamma C}{(1-\gamma)^2}\expect_{s \sim d^{\pi^-}}\left[D_{TV}(\pi^-(s)\| \pi_h(s))\right].\\
\end{aligned}
\end{equation}

By the definition of $\beta_1$ in Lemma~\ref{lemma-1}:
\begin{equation}
\begin{aligned}
&\expect_{s \sim d^{\pi^-},~a \sim \pi_h}\left[A^{\pi^-}(s,a)\right]\\ =& \expect_{s \sim d^{\pi^-},~a \sim  \pi^-}\left[
(\frac{\pi_h(a|s)}{\pi^-(a|s)}- 1)A^{\pi^-}(s,a)
\right]\\
\leq& \beta_1 \expect_{s \sim d^{\pi^-},~a\sim \pi^-}\left[A^{\pi^-}(s,a)\right]
 \leq \beta_1C
\end{aligned}
\nonumber
\end{equation}

Combining this and the definition of $C$ and $\beta$ with
inequality~(\ref{lemma-1-2}), we get the bound in Lemma~\ref{lemma-1}.
\end{proof}

In~\cite{1}, the authors prove the relation between the expected KL-divergence and the expected TV-divergence of the distribution $p$ and $q$ on state $s$ satisfies: 
\begin{equation}
    \expect_{s \sim f(s)}{D_{TV}(p(\cdot|s) \| q(\cdot|s))} \leq \expect_{s \sim f(s)}\sqrt{D_{KL}(p(\cdot|s) \| q(\cdot|s))/2}, \nonumber
\end{equation}
where $f(s)$ is the distribution on state $s$. Therefore the expected TV-divergence can be bounded by KL-divergence.

\begin{lemma}[\textbf{The adversary is stronger with a stronger adversarial optimizer}]
\label{lemma-2} 
We can bound the objective of the original problem~(8):
\begin{equation}
    \expect_{s \sim d^{\pi^-}} \left[ D_{TV}(\pi_h(\cdot|s)\| \pi^-(\cdot|s)) \right]\leq \sqrt{\beta_0/2}, \nonumber
\end{equation}
here $\beta_0 = \max_{s \in S}\left\|D_{KL}(\pi_h(\cdot|s) \| \pi^-(\cdot|s))\right\| $.
\end{lemma}
Lemma~\ref{lemma-2} shows that the bound of the objective in problem~(9) is closely related to the optimization method solving problem~(10). 
With Lemma~\ref{lemma-1} and Lemma~\ref{lemma-2}, we further provide an upper bound of the performance after attack by $\hat{\alpha}$-adversary.

\begin{lemma}[\textbf{Upper bound of the \bm{$\hat{\alpha}$}-adversary's performance}]
\label{lemma-3}
Let the adversary be an $\hat{\alpha}$-adversary. The performance of the perturbed policy $\pi_h$ satisfies:
\begin{equation}
    R(\pi_h) \leq \hat{\alpha} + \frac{C\beta_1}{1-\gamma} + \frac{2\gamma C\sqrt{\beta_0/2}}{(1-\gamma)^2}+R(\pi^-), \nonumber
\end{equation}
where $C$, $\beta_0$ and $\beta_1$ are defined in Lemma~\ref{lemma-1} and Lemma~\ref{lemma-2}. 
\end{lemma}

Lemma \ref{lemma-3} implies that the performance of the adversarial attack is bounded by the ability $\alpha$ of $\alpha$-adversary and the distance from policy $\pi_h$ and $\pi^-$.

\begin{thm}[\textbf{\bm{$\hat{\alpha}$}-adversary is stronger than other adversary under some conditions}]
\label{thm-4}
Let $e$ be an arbitrary adversarial attack algorithm, set $\alpha_{e} = R(\pi_e) - R(\pi^-)$ and $\beta_1 = \max_{s, a} \|\frac{\pi_h(a|s)}{\pi^-(a|s)} - 1\|$. If $\beta_1$ satisfies:

\begin{equation}
     \beta_1 < \frac{-\sqrt{2}\gamma C + \sqrt{2\gamma^2C^2+4(\alpha_e - \hat{\alpha})(1-\gamma)^3}}{2(1-\gamma)C}, \nonumber
\end{equation}
then the performance of the victim policy after our algorithm attack satisfies: $R(\pi_h) < R(\pi_{e})$. In other words, our attack is stronger than adversarial attack $e$. 
\end{thm}

\begin{proof}
Let $p(a) = \pi_h(a|s)$, $q(a) = \pi^-(a|s)$. then:
\begin{equation}
    \sum_a  p(a) \ln(\frac{p(a)}{q(a)}) \leq \sum_a p(a) \ln(1+\beta_1) \leq \beta_1, \nonumber
\end{equation}
with the inequality $\ln(1+x) \le x$ when $x \ge 0$. Therefore, $\beta_0 \le \beta_1$, which bounds the performance of policy $\pi_h$:
\begin{equation}
    \label{inequal: last}
    \begin{aligned}
    R(\pi_h) &\leq \hat{\alpha} + \frac{C\beta_1}{1-\gamma} + \frac{2\gamma C\sqrt{\beta_0/2}}{(1-\gamma)^2} \\
    & \leq \hat{\alpha} + \frac{C\beta_1} {1-\gamma} + \frac{2C \gamma\sqrt{\beta_1/2}}{(1-\gamma)^2}.
\end{aligned}
\end{equation}
\end{proof}

\end{appendix}


\end{document}